\pdfoutput=1
\documentclass[10pt]{scrartcl}

\usepackage[english]{babel} 
\usepackage[protrusion=true,expansion=true]{microtype} 
\usepackage{amsmath,amsfonts,amsthm,amssymb,color,epsfig,fullpage} 
\usepackage{verbatim,fancyvrb}
\usepackage{alltt}
\usepackage{tikz}
\usepackage{enumerate}
\usepackage{caption}
\usepackage{subfigure}
\usepackage[]{algorithm}
\usepackage{algpseudocode}
\usepackage{stmaryrd}
\usepackage{afterpage}
\usepackage{url}

\newtheorem{definition}{Definition}
\newtheorem{theorem}[definition]{Theorem}
\newtheorem{corollary}[definition]{Corollary}
\newtheorem{proposition}[definition]{Proposition}
\newtheorem{lemma}[definition]{Lemma}

\newtheorem{remark}[definition]{Remark}

\def\N{{\mathbb N}}

\def\R{{\mathbb R}}

\DeclareMathOperator*{\argmin}{argmin}
\DeclareMathOperator{\dtw}{\delta}

\newcommand*{\tran}{^{\mkern-1.5mu\mathsf{T}}}

\newcommand{\commentout}[1]{}

\newcommand{\abs}[1]{\mathop{\left\lvert #1 \right\rvert}} 
\newcommand{\args}[1]{\mathop{\left( #1 \right)}} 

\newcommand{\norm}[1]{\mathop{\left\lVert #1 \right\rVert}}
\newcommand{\cbrace}[1]{\mathop{\left\{ #1 \right\}}}

\newcommand{\argsS}[2]{\mathop{\left( #1 \right)#2}} 

\newcommand{\normS}[2]{\mathop{\left\lVert #1 \right\rVert#2}}

\newcommand{\T}{\mathop{\mathsf{T}}}           	

\renewcommand{\S}[1]{{\mathcal{#1}}}           	

\renewenvironment{cases}{%
\left\{\begin{array}{c@{\quad : \quad}l}}%
{%
\end{array}\right.}

\begin{document}

\title{Optimal Warping Paths are unique for almost every Pair of Time Series}

\author{Brijnesh J.~Jain and David Schultz\\
       Technische Universit\"at Berlin, Germany\\
       e-mail: brijnesh.jain@gmail.com}
\date{}
\maketitle

\begin{abstract} 
Update rules for learning in dynamic time warping spaces are based on optimal warping paths between parameter and input time series. In general, optimal warping paths are not unique resulting in adverse effects in theory and practice. Under the assumption of squared error local costs, we show that no two warping paths have identical costs almost everywhere in a measure-theoretic sense. Two direct consequences of this result are: (i) optimal warping paths are unique almost everywhere, and (ii) the set of all pairs of time series with multiple equal-cost warping paths coincides with the union of exponentially many zero sets of quadratic forms. One implication of the proposed results is that typical distance-based cost functions such as the k-means objective are differentiable almost everywhere and can be minimized by subgradient methods. 
\end{abstract}


\section{Introduction}

\subsection{Dynamic time warping}

Time series such as audio, video, and other sensory signals represent a collection of time-dependent values that may vary in speed (see Fig.~\ref{fig:euclid}). Since the Euclidean distance is sensitive to such variations, its application to time series related data mining tasks may give unsatisfactory results \cite{Esling2012,Fu2011,Xing2010}. Consequently, the preferred approaches to compare time series apply elastic transformations that filter out the variations in speed. Among various techniques, one of the most common elastic transformation is dynamic time warping (DTW) \cite{Sakoe1978}.

\begin{figure}[t]
\centering
 \includegraphics[width=0.47\textwidth]{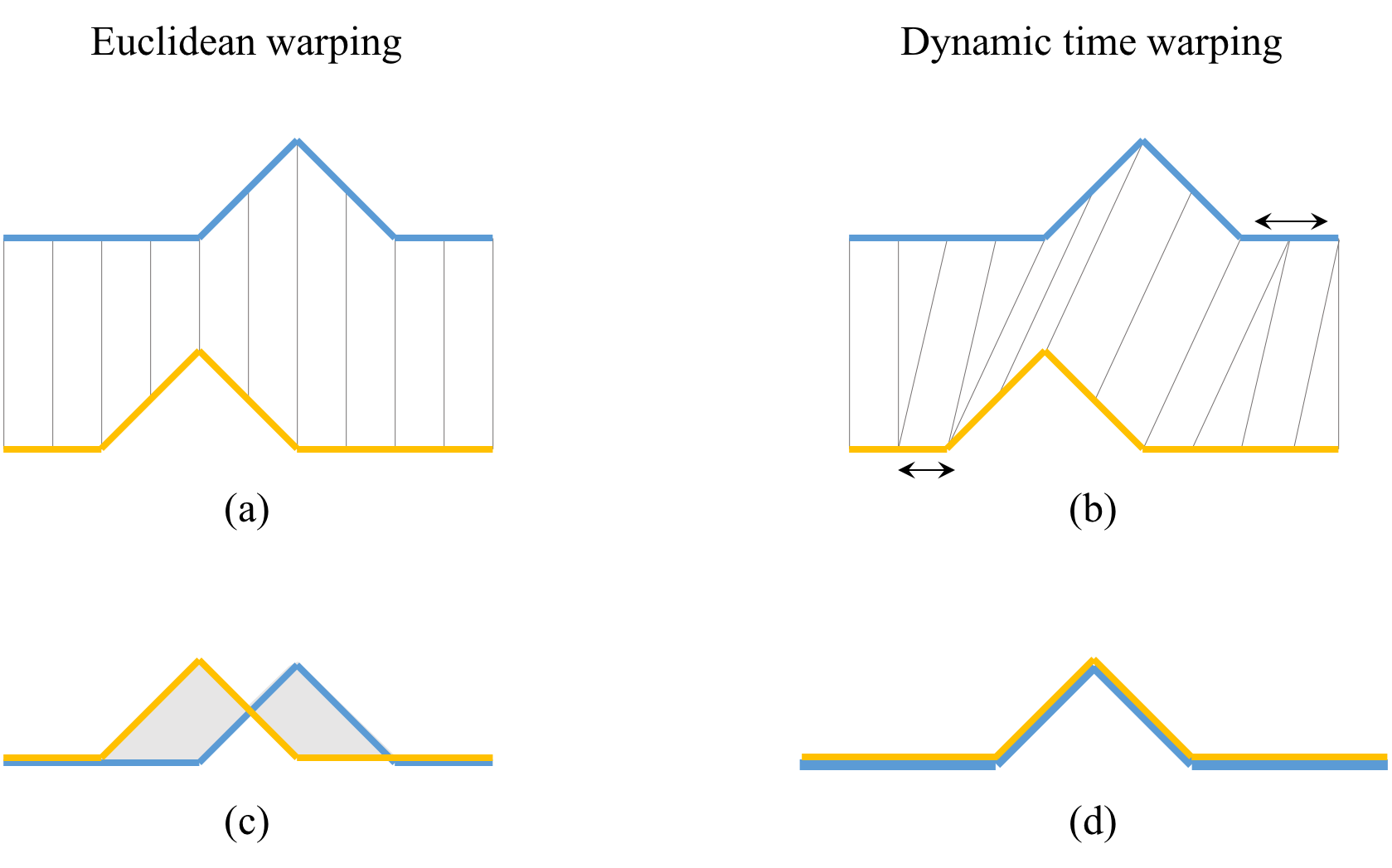}
\caption{Variations in speed and their implications for comparing two time series shown in blue and yellow. Both time series are similar in shape but differ in speed along their flat regions. Plot (a): The Euclidean distance warps the $i$-th points of both time series onto one another as shown by the gray lines. Since the Euclidean warping is sensitive to variations in speed the shapes are not preserved and both peaks are aligned to points in flat regions. Plot (c): Euclidean warping leaves the time series unchanged and results in a large dissimilarity score as indicated by the gray shaded area. Plot (b): An optimal warping that preserves the shape of both time series. Double arrows indicate segments stretched by dynamic time warping. Plot (d): Warped time series obtained by the optimal warping in (b). The left flat segment of the yellow and the right flat segment of the blue time series have been stretched to align both time series. The resulting dissimilarity is zero and better reflects the similarity in shape than the Euclidean distance. }
\label{fig:euclid} 
\end{figure}

Dynamic time warping is based on the concept of warping path. A warping path $p$ determines how to stretch two given time series $x$ and $y$ to warped time series $x'$ and $y'$ under certain constraints. The cost of warping $x$ and $y$ along warping path $p$ measures how dissimilar the warped time series $x'$ and $y'$ are. There are exponential many different warping paths \cite{Banderier2005} each of which determines the cost of warping time series $x$ and $y$. An optimal warping path of $x$ and $y$ is a warping path with minimum cost. Optimal warping paths exist but are not unique in general (see Fig.~\ref{fig:non-unique}).


\subsection{The problem of non-uniqueness}

Recent research is directed towards extending standard statistical concepts and machine learning methods to time series spaces endowed with the DTW distance. Examples include time series averaging \cite{Brill2018,Cuturi2017,Kruskal1983,Petitjean2011,Schultz2017}, k-means clustering \cite{Hautamaki2008,Petitjean2016,Soheily-Khah2015}, self-organizing maps \cite{Kohonen1998}, learning vector quantization \cite{Somervuo1999,Jain2017}, and warped-linear classifiers \cite{Jain2015,Jain2017b}. 

The lowest common denominator of these approaches is that they repeatedly update one or more parameter time series. In addition, update directions are based on optimal warping paths such that the following properties hold:
\begin{itemize}
\itemsep0em
\item
If an optimal warping path is unique, then the update direction is well-defined.
\item 
If an optimal warping path is non-unique, then there are several update directions. 
\end{itemize}
Non-uniqueness of optimal warping paths complicates the algorithmic design of learning methods in DTW spaces and their theoretical analysis. In some situations, non-uniqueness may result in adverse effects. For example, repulsive updating in learning vector quantization finds a theoretical justification only in cases where the corresponding optimal warping path is unique \cite{Jain2017}.

Given the problems caused by non-uniqueness, it is desirable that optimal warping paths are unique almost everywhere. In this case, non-unique optimal warping paths occur exceptionally and are easier to handle as we will shortly. Therefore, we are interested in how prevalent unique optimal warping paths are.

\subsection{Almost everywhere}

The colloquial term ``almost everywhere'' has a precise measure-theoretic meaning. A measure quantifies the size of a set. It generalizes the concepts of length, area, and volume of a solid body defined in one, two, and three dimensions, respectively. The term ``almost everywhere'' finds its roots in the notion of a ``negligible set''. Negligible sets are sets contained in a set of measure zero. For example, the function 
\[
f(x) = \begin{cases}
1 & x \neq 0\\
0 & x = 0
\end{cases}
\]
is discontinuous on the negligible set $\cbrace{0}$ with measure zero. We say, function $f$ is continuous almost everywhere, because the set where $f$ is not continuous is negligible. More generally, a property $P$ is said to be true almost everywhere if the set where $P$ is false is negligible. The property that an optimal warping is unique almost everywhere means that the set of all pairs of time series with non-unique optimal warping path is negligible. 

When working in a measure space, a negligible set contains the exceptional cases we can handle or even do not care about and often ignore. For example, we do not care about the behavior of the above function $f$ on its negligible set $\cbrace{0}$ when computing its Lebesgue integral over $[-1, 1]$. Another example is that cost functions of some machine learning methods in Euclidean spaces such as k-means or learning vector quantization are non-differentiable on a negligible set. In such cases, it is common practice to ignore such points or to resort to subgradient methods.

\subsection{Contributions}

Consider the following property $P$ on the set $\S{F}^m \times \S{F}^n$ of pairs of time series of length $m$ and $n$: 
The pair of time series $x \in \S{F}^m$ and $y \in \S{F}^n$ satisfies $P$ if there are two different (not necessarily optimal) warping paths between $x$ and $y$ with identical costs. Under the assumption of a squared error local cost function, the main result of this article is Theorem \ref{theorem:ae}: 
\begin{quote}\em
Property $P$ is negligible on $\S{F}^m \times \S{F}^n$. 
\end{quote}
Direct consequences of Theorem \ref{theorem:ae} are (i) optimal warping paths are unique almost everywhere, and (ii) property $P$ holds on the union of exponentially many zero sets of quadratic forms. The results hold for uni- as well as multivariate time series. 

An implication of non-unique optimal warping paths is that adverse effects in learning are exceptional cases that can be safely handled. For example, learning amounts in (stochastic) gradient descent update rules almost everywhere.

\begin{figure}[t]
\centering
 \includegraphics[width=0.8\textwidth]{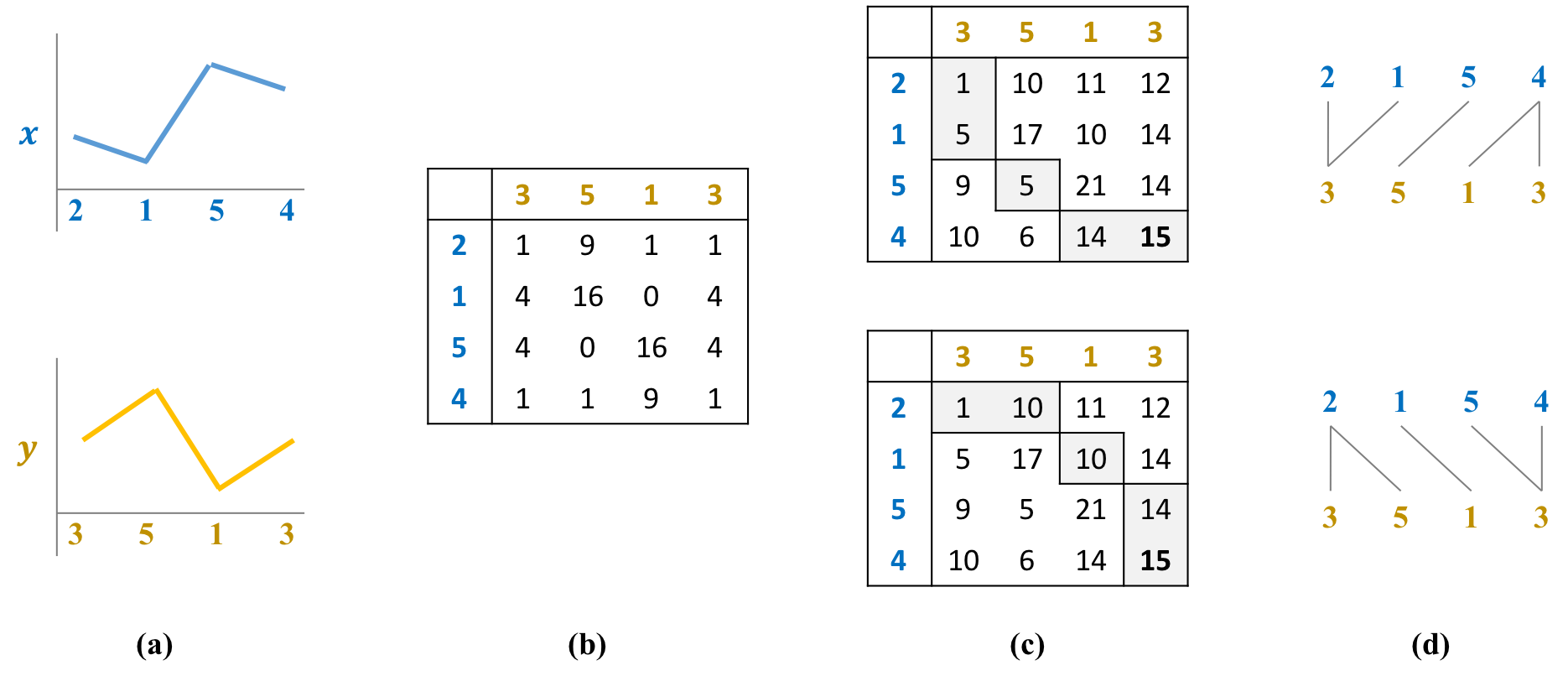}
\caption{Example of non-unique warping paths (see Section \ref{sec:warping-paths} for definitions). Section (a) Time series $x$ and $y$. (b) Local cost matrix with elements $(x_i-y_j)^2$. (c) Accumulated cost obtained by dynamic programming \cite{Sakoe1978} gives the DTW distance $\dtw(x, y) = \sqrt{15}$. The gray shaded cells show two different optimal warping paths. (d) Alignments of $x$ and $y$ by the two different optimal warping paths. }
\label{fig:non-unique} 
\end{figure}

\section{Background}

This section first introduces warping paths and then defines the notions of \emph{negligible} and \emph{almost everywhere} from measure theory.

\subsection{Time Series and Warping Paths}\label{sec:warping-paths}

We first define time series. Let $\S{F} = \R^d$ denote the $d$-dimensional Euclidean space. A $d$-variate \emph{time series} of length $m$ is a sequence $x = (x_1, \ldots, x_m)$ consisting of elements $x_i \in \S{F}$. By $\S{F}^m$ we denote the set of all time series of length $m$ with elements from $\S{F}$. 

\medskip

Next, we describe warping paths. Let $[n] = \cbrace{1, \ldots, n}$, where $n \in \N$. An ($m \times n$)-lattice is a set of the form $\S{L}_{m,n} = [m] \times [n]$. A \emph{warping path} in lattice $\S{L}_{m,n}$ is a sequence $p = (p_1 , \dots, p_L)$ of $L$ points $p_l = (i_l,j_l) \in \S{L}_{m,n}$ such that
\begin{enumerate}
\item $p_1 = (1,1)$ and $p_L = (m,n)$ 
\item $p_{l+1} - p_{l} \in \cbrace{(1,0), (0,1), (1,1)}$ for all $l \in [L-1]$.
\end{enumerate}
The first condition is called boundary condition and the second one is the step condition. 

By $\S{P}_{m,n}$ we denote the set of all warping paths in $\S{L}_{m,n}$. A warping path departs at the upper left corner $(1,1)$ and ends at the lower right corner $(m, n)$ of the lattice. Only east $(0, 1)$, south $(1, 0)$, and southeast $(1, 1)$ steps are allowed to move from a given point $p_l$ to the next point $p_{l+1}$ for all $1 \leq l < L$.

\medskip

Finally, we introduce optimal warping paths. A warping path $p \in \S{P}_{m,n}$ defines an alignment (warping) between time series $x \in \S{F}^m$ and $y \in \S{F}^n$ by relating elements $x_{i}$ and $y_{j}$ if $(i, j) \in p$. The \emph{cost} of aligning time series $x$ and $y$ along warping path $p$ is defined by
\[
C_p(x, y) = \sum_{(i,j) \in p} \normS{x_i-y_j}{^2},
\]
where $\norm{\cdot}$ denotes the Euclidean norm on $\S{F}$. A warping path $p_*\in \S{P}_{m,n}$ between $x$ and $y$ is \emph{optimal} if 
\[
p_* \in \argmin_{p \in \S{P}_{m,n}} C_p(x, y).
\]
By $\S{P}_*(x,y)$ we denote the set of all optimal warping paths between time series $x$ and $y$. The DTW distance is defined by
\[
\delta(x, y) = \min_{\displaystyle p \in \S{P}_{m,n}} \sqrt{C_p(x, y)}.
\]
is the DTW distance. Observe that $\delta(x, y) = \sqrt{C_p(x, y)}$ for all $p \in \S{P}_*(x,y)$.

\subsection{Measure-Theoretic Concepts}

We introduce the necessary measure-theoretic concepts to define the notions of \emph{negligible} and \emph{almost everywhere}. For details, we refer to \cite{Halmos2013}. 

\medskip

One issue in measure theory is that not every subset of a given set $\S{X}$ is measurable. A family $\S{A}$ of measurable subsets of a set $\S{X}$ is called $\sigma$-algebra in $\S{X}$. A measure is a function $\mu: \S{A} \rightarrow \R_+$ that assigns a non-negative value to every measurable subset of $\S{X}$ such that certain conditions are satisfied. To introduce these concepts formally, we assume that $\mathfrak{P}(\S{X})$ denotes the power set of a set $\S{X}$, that is the set of all subsets of $\S{X}$. A system $\S{A} \subset \mathfrak{P}(\S{X})$ is called a 
$\sigma$\emph{-algebra} in $\S{X}$ if it has the following properties:
\begin{enumerate}
\item $\S{X} \in \S{A}$
\item $\S{U} \in \S{A}$ implies $\S{X} \setminus \S{U} \in \S{A}$
\item $\args{\S{U}_i}_{i \in \N} \in \S{A}$ implies $\bigcup_{i \in \N} \S{U}_i \in \S{A}$.
\end{enumerate}
A \emph{measure} on $\S{A}$ is a function $\mu: \S{A} \rightarrow [0, +\infty]$ that satisfies the following properties:
\begin{enumerate}
\item $\mu(\S{U}) \geq 0$ for all $\S{U} \in \S{A}$
\item $\mu(\emptyset) = 0$
\item For a countable collection of disjoint sets $\args{\S{U}_i}_{i \in \N} \in \S{A}$, we have
\[
\mu\args{\bigcup_{i \in \N} \S{U}_i} = \sum_{i \in \N} \mu\args{\S{U}_i}.
\]
\end{enumerate}
A triple $(\S{X}, \S{A}, \mu)$ consisting of a set $\S{X}$, a $\sigma$-algebra $\S{A}$ in $\S{X}$ and a measure $\mu$ on $\S{A}$ is called a \emph{measure space}. The Borel-algebra $\S{B}$ in $\R^d$ is the $\sigma$-algebra generated by the open sets of $\R^d$. The Lebesgue-measure $\mu$ on $\S{B}$ generalizes the concept of $d$-volume of a box in $\R^d$. The triple $\args{\R^d, \S{B}, \mu}$ is called \emph{Borel-Lebesgue measure space}. 

Let $\args{\S{X}, \S{A}, \mu}$ be a measure space, where $\S{X}$ is a set, $\S{A}$ is a $\sigma$-algebra in $\S{X}$, and $\mu$ is a measure defined on $\S{A}$. A set $\S{N} \subset \S{X}$ is $\mu$-\emph{negligible} if there is a set $\S{N}' \in \S{A}$ such that $\mu(\S{N}') = 0$ and $\S{N} \subseteq \S{N}'$. A property of $\S{X}$ is said to hold $\mu$-\emph{almost everywhere} if the set of points in $\S{X}$ where this property fails is $\mu$-negligible.

\section{Results}

We first show that optimal warping paths are unique almost everywhere. Then we geometrically describe the location of the non-unique set. Finally, we discuss the implications of the proposed results on learning in DTW spaces.

\medskip

Let $\S{X} = \S{F}^m \times \S{F}^n$ bet the set of all pairs $(x, y)$ of time series, where $x$ has length $m \in \N$ and $y$ has length $n\in \N$. We regard the set $\S{X}$ as a Euclidean space and assume the Lebesgue-Borel measure space $\args{\S{X}, \S{B}, \mu}$.\footnote{See Remark \ref{remark:justification} for an explanation of why we regard $\S{X}$ as a Euclidean space.} 
The \emph{multi optimal-path set} of $\S{X}$ is defined by
\[
\S{N}_{\S{X}}^* = \cbrace{(x, y) \in \S{X} \,:\, \abs{\S{P}_*(x,y)} > 1}.
\]
This set consists of all pairs $(x, y) \in \S{X}$ with non-unique optimal warping path. To assert that $\S{N}_{\S{X}}^*$ is $\mu$-negligible, we show that $\S{N}_{\S{X}}^*$ is a subset of a set of measure zero. For this, consider the \emph{multi-path set} 
\[
\S{N_X} = \bigcup_{\substack{p,q \in \S{P}_{m,n}\\ p \neq q}} \cbrace{(x, y) \in \S{X} \,:\, C_p(x, y) = C_q(x, y) },
\]
The set $\S{N_X}$ consists of all pairs $(x, y)$ that can be aligned along different warping paths with identical cost. Obviously, the set $\S{N}_{\S{X}}^*$ is a subset of $\S{N_X}$. The next theorem states that $\S{N_X}$ is a set of measure zero.
\begin{theorem}\label{theorem:ae}
Let $\args{\S{X}, \S{B}, \mu}$ be the Lebesgue-Borel measure space. Then $\mu(\S{N_X}) = 0$.
\end{theorem}

From Theorem \ref{theorem:ae} and $\S{N}_{\S{X}}^* \subseteq \S{N_X}$ immediately follows that $\S{N}_{\S{X}}^*$ is $\mu$-negligible.

\begin{corollary}
Under the assumptions of Theorem \ref{theorem:ae} the set $\S{N}_{\S{X}}^*$ is $\mu$-negligible.
\end{corollary}

Thus, optimal warping paths are unique $\mu$-almost everywhere in $\S{X}$. Even more generally: The property that all warping paths have different cost holds $\mu$-almost everywhere in $\S{X}$.

\medskip

We describe the geometric form of the multi-path set $\S{N_X}$. For this we identify $\S{F}^m \times \S{F}^n$ with $\S{F}^k$, where $k = m+n$. Thus, pairs $(x, y) \in \S{X}$ of time series are summarized to $z \in \S{F}^k$, henceforth denoted as $z \in \S{X}$. By $\S{X}^2 = \S{F}^{k \times k}$ we denote the set of all ($k \times k$)-matrices with elements from $\S{F}$. Finally, 
the zero set of a function $f:\S{X} \rightarrow \R$ is of the form
\[
\S{Z}(f) = \cbrace{z \in \S{X} \,:\, f(z) = 0}.
\]
From the proof of Theorem \ref{theorem:ae} directly follows that the set $\S{N_X}$ is the union of zero sets of quadratic forms. 
\begin{corollary}\label{cor:form}
Under the assumptions of Theorem \ref{theorem:ae}, there is an integer $D \in \N$ and symmetric matrices $A_1, \ldots, A_D \in \S{X}^2$ such that
\[
\S{N_X} = \bigcup_{i=1}^D \S{Z}\args{z\tran A_i \, z}.
\]
\end{corollary}

\medskip

The number $D$ of zero sets in Corollary \ref{cor:form} grows exponentially in $m$ and $n$. From the proof of Theorem \ref{theorem:ae} follows that $D = D_{m,n}(D_{m,n}-1)/2$, where 
\[
D_{m,n} = \sum_{i=0}^{\min\cbrace{m,n}} 2^i \binom{m}{i}\binom{n}{i}.
\]
is the Delannoy number \cite{Banderier2005}. The Delannoy number $D_{m,n} = \abs{\S{P}_{m,n}}$ counts the number of all warping paths in lattice $\S{L}_{m,n}$. Table \ref{tab:Delannoy} presents the first Delannoy numbers up to $m = 10$ and $n = 9$. We see that there are more than half a million warping paths in a ($10\times 9$)-lattice showing that $\S{N}_X$ is the union of more than 178 billion zero sets. For two time series of length $20$, the number of warping paths is $D_{20,20} = 260,543,813,797,441$, which is more than $260$ trillions. Thus, the multi-path set $\S{N_X}$ of two time series of length $20$ is the union of more than $33$ octillion zero sets, that is $d > 33 \cdot 10^{27}$. An open question is the number $D^* \leq D$ of zero sets that form the multi optimal-path set $\S{N}_{\S{X}}^*$. The example in Figure \ref{fig:plot} indicates that the multi optimal-path set $\S{N}_{\S{X}}^*$ can be much smaller than the multi-path set $\S{N_X}$.

\begin{table}[t]
\small
\centering
\begin{tabular}{r|rrrrrrrrr}
$m \setminus n$ & 1 & 2 & 3 & 4 & 5 & 6 & 7 & 8 & 9\\
\hline
&&&&&&&&&\\[-2ex]
1 & 1 & 1 & 1 & 1 & 1 & 1 & 1 & 1 & 1 \\
2 & 1 & 3 & 5 & 7 & 9 & 11 & 13 & 15 & 17 \\
3 & 1 & 5 & 13 & 25 & 41 & 61 & 85 & 113 & 145 \\
4 & 1 & 7 & 25 & 63 & 129 & 231 & 377 & 575 & 833 \\
5 & 1 & 9 & 41 & 129 & 321 & 681 & 1,289 & 2,241 & 3,649 \\
6 & 1 & 11 & 61 & 231 & 681 & 1,683 & 3,653 & 7,183 & 13,073 \\
7 & 1 & 13 & 85 & 377 & 1,289 & 3,653 & 8,989 & 19,825 & 40,081 \\
8 & 1 & 15 & 113 & 575 & 2,241 & 7,183 & 19,825 & 48,639 & 108,545 \\
9 & 1 & 17 & 145 & 833 & 3,649 & 13,073 & 40,081 & 108,545 & 265,729 \\
10 & 1 & 19 & 181 & 1,159 & 5,641 & 22,363 & 75,517 & 224,143 & 598,417 \\
\end{tabular}
\caption{Delannoy numbers $D_{m,n}$ for $m \in [10]$ (rows) and $n \in [9]$ (columns).}
\label{tab:Delannoy}
\end{table}

\begin{figure}[t]
\centering
\includegraphics[width=0.5\textwidth]{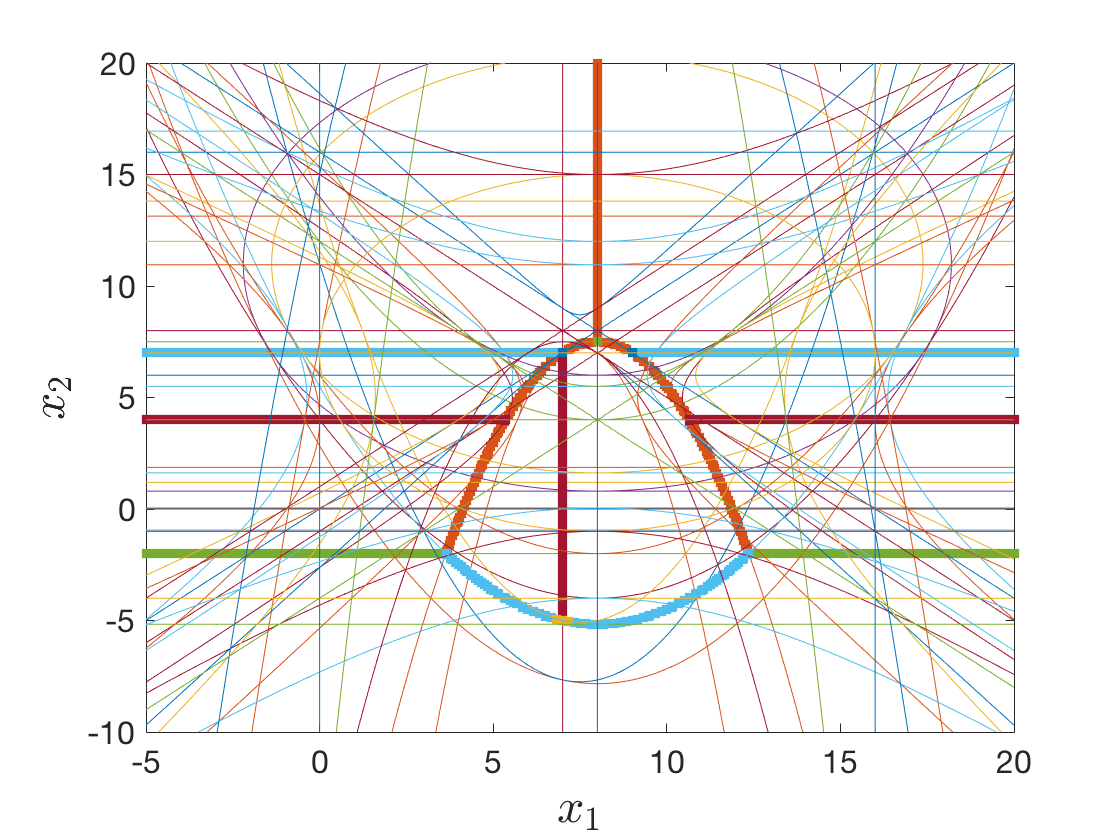}
\caption{Illustration of the multi-path set $\S{N_X}$ and multi optimal-path sets $\S{N}_{\S{X}}^*$ for the pair of time series $x = (x_1, x_2, 0)$ and $y = (4,8,7)$ with varying values for $x_1$ and $x_2$. The Delannoy number $D_{3,3}$ is $13$. Hence, the multi-path set $\S{N_X}$ consists of $D = 78$ zero sets of quadratic forms as indicated by colored curves. The subset of multi optimal-path set $\S{N}_{\S{X}}^*$ is highlighted by fat curve segments.}
\label{fig:plot} 
\end{figure}

\subsection{Discussion}

We discuss the implications of Theorem \ref{theorem:ae} for learning in DTW spaces. 

\subsubsection{Learning}

This section shows that almost-everywhere uniqueness implies almost-everywhere differentiability of the underlying cost function.

\medskip

To convey the line of argument, it is sufficient to restrict to the problem of averaging time series as representative for other, more complex learning problems. In contrast to computing the average in Euclidean spaces, time series averaging is a non-trivial task for which the complexity class is currently unknown \cite{Brill2018}.

Let $x^{(1)}, \ldots, x^{(N)}$ be a sample of $N$ time series, possibly of varying length. Consider the cost function $J: \S{F}^n \rightarrow \R$ of the form
\[
J(z) = \frac{1}{N}\sum_{k=1}^N \ell \args{\delta\args{x^{(k)}, z}},
\]
where $\ell: \R \rightarrow \R$ is a loss function. Common loss functions for averaging time series are the identity $\ell(a) = a$ and the squared loss $\ell(a) = a^2$. The goal of time series averaging is to find a time series $z_* \in \S{F}^n$ of length $n$ that minimizes the cost $J(z)$.

The challenge of time series averaging is to minimize the non-differentiable cost function $J(z)$. We show that almost-everywhere uniqueness of optimal warping paths implies almost-everywhere differentiability of the cost function $J(z)$ and provides a stochastic (incremental) update rule. 

We express the DTW distance $\delta(x, z)$ as a parametrized function. Suppose that $x \in \S{F}^m$ is a time series. Then the parametrized DTW function $\delta_x: \S{F}^n \rightarrow \R$ restricted to the set $\S{F}^n$ is of the form $\delta_x(z) = \delta(x, z)$. We have the following result: 

\begin{proposition}\label{prop:differentiability}
Suppose that $x \in \S{F}^m$ and $z \in \S{F}^n$ are two time series with unique optimal warping path. Then the function $\delta_x(z) = \delta(x, z)$ is differentiable at $z$ and its gradient is a time series of the form 
\[
\nabla_z \delta_x(z) = \nabla_z C_{p_*}(x, z) \in \S{F}^n,
\]
where $p_* \in \S{P}_*(x, z)$ is an optimal warping path. 
\end{proposition}

The proof follows from \cite{Rockafellar2003} after reducing $\delta_x(z)$ to a piecewise smooth function. By construction, the cost $C_p(x, z)$ of warping $x$ and $z$ along warping path $p \in \S{P}_{m,n}$ is differentiable as a function of the second argument $z \in \S{F}^n$. Non-differentiability of $\delta_x(z)$ is caused by non-uniqueness of an optimal warping path between $x$ and $z$. In this case we have 
\[
\delta_x(z) = C_{p_*}(x, z) = C_{q_*}(x, z),
\]
where $p_*, q_* \in \S{P}_*(x, z)$ are two distinct optimal warping paths. Then it can happen that 
 \[
\nabla_z C_{p_*}(x, z) \neq \nabla_y C_{q_*}(x, z)
\]
showing that $\delta_x$ is non-differentiable at $z$. 

Next, suppose that the loss function $\ell: \R \rightarrow \R$ is differentiable and an optimal warping path $p_*$ between time series $x$ and $z$ is unique. Then the individual cost $J_x(z) = \ell\args{\delta(x, z)}$ is also differentiable at $z$ with gradient 
\begin{align*}
\nabla_z J_x(z) = \ell'(C_{p_*}(x, z))\nabla_z C_{p_*}(x, z).
\end{align*}
Differentiability of $J_x(z)$ gives rise to a stochastic (incremental) update rule of the form
\begin{align}\label{eq:stochastic-updating}
z \leftarrow z - \eta\, \ell'(C_{p_*}(x, z))\nabla_z C_{p_*}(x, z),
\end{align}
where $\eta$ is the step size and $p_*$ is an optimal warping path between $x$ and $z$. We can also apply the stochastic update rule \eqref{eq:stochastic-updating} in cases where an optimal warping path between $x$ and $z$ is not unique. In this case, we first (randomly) select an optimal warping path $p_*$ from the set $\S{P}_*(x, z)$. Then we update $z$ by applying update rule \eqref{eq:stochastic-updating}.

Updating at non-differentiable points according to the rule \eqref{eq:stochastic-updating} is not well-defined. In addition, it is unclear whether the update directions are always directions of descent, for learning problems in general. The next result confines both issues to a negligible set. 

\begin{corollary}\label{cor:diffentiable}
Suppose that $\S{X} = \S{F}^m \times \S{F}^n$ and $\ell: \R \rightarrow \R$ is a differentiable loss function. Then the functions
\begin{itemize}
\itemsep0em
\item 
$\delta(x, z)= \delta_x(z)$ 
\item 
$J_x(z) = \ell\args{\delta(x, z)}$
\end{itemize}
are differentiable $\mu$-almost everywhere on $\S{X}$. 
\end{corollary}

Corollary \ref{cor:diffentiable} directly follows from Prop.~\ref{prop:differentiability} together with Theorem \ref{theorem:ae}. In summary, almost-everywhere uniqueness of optimal warping paths implies almost-everywhere differentiability of the individual cost $J_x(z)$. The latter in turn implies that update rule \eqref{eq:stochastic-updating} is a well-defined stochastic gradient step almost everywhere. 

The arguments in this section essentially carry over to other learning problems in DTW spaces such as k-means, self-organizing maps, and learning vector quantization. We assume that it should not be a problem to transfer the proposed results to learning based on DTW similarity scores as applied in warped-linar classifiers \cite{Jain2015,Jain2017b}.

\subsubsection{Learning Vector Quantization in DTW Spaces}

Learning vector quantization (LVQ) is a supervised classification scheme introduced by Kohonen \cite{Kohonen2001}. A basic principle shared by most LVQ variants is the margin-growth principle \cite{Jain2017}. This principle justifies the different learning rules and corresponds to stochastic gradient update rules if a differentiable cost function exists. As the k-means algorithm, the LVQ scheme has been generalized to DTW spaces \cite{Jain2017,Somervuo1999}. In this section, we illustrate that a unique optimal-warping path is a necessary condition to satisfy the margin-growth principle in DTW spaces as proved in \cite{Jain2017}, Theorem 12.

\medskip

As a representative example, we describe LVQ1, the simplest of all LVQ algorithms \cite{Kohonen2001}. Let $\S{X} = \R^d$ be the $d$-dimensional Euclidean space and let $\S{Y} = \cbrace{1, \ldots, C}$ be a set consisting of $C$ class labels. The LVQ1 algorithm assumes a codebook $\S{C} = \cbrace{(p_1,z_1), \ldots, (p_K, z_k)}$ of $K$ prototypes $p_k \in \S{X}$ with corresponding class labels $z_k \in \S{Y}$. As a classifier, LVQ1 assigns an input point $x \in \S{X}$ to the class $z_c$ of its closest prototype $p_c \in \S{C}$, where 
\[
c \in \argmin_k \normS{p_k - x}{^2}.
\]
LVQ1 learns a codebook $\S{C}$ on the basis of a training set $\S{D} = \cbrace{\args{x_1, y_1}, \ldots, \args{x_N, y_N}} \subseteq \S{X} \times \S{Y}$. After initialization of $\S{C}$, the algorithm repeats the following steps until termination: (i) Randomly select a training example $(x_i, y_i) \in \S{D}$; (ii) determine the prototype $p_c \in \S{C}$ closest to $x_i$; and (iii) attract $p_c$ to $x_i$ if their class labels agree and repel $p_c$ from $x_i$ otherwise. Step (iii) adjusts $p_c$ according to the rule
\[
p_c \leftarrow p_c \pm \eta\args{x_i - p_c},
\]
where $\eta$ is the learning rate, the sign $\pm$ is positive if the labels of $p_c$ and $x_i$ agree ($z_c = y_i$), and negative otherwise ($z_c \neq y_i$). The update rule guarantees that adjusting $p_c$ makes an incorrect classification of $x_i$ more insecure. Formally, if the learning rate $\eta < \theta$ is bounded by some threshold $\theta$, the LVQ1 update rule guarantees to increase the hypothesis margin 
\[
\mu_c(x_i) = \begin{cases}
\normS{x_i - p^-}{^2} - \: \normS{x_i - p_c}{^2} & y_i = z_c\\[1ex]
\normS{x_i - p_c}{^2} -\:\normS{x_i - p^+}{^2} & y_i \neq z_c
\end{cases},
\]
where $p^+$ ($p^-$) is the closest prototypes of $x$ with the same (different) class label.

The different variants of LVQ have been extended to DTW spaces by replacing the squared Euclidean distance with the squared DTW distance \cite{Jain2017,Somervuo1999}. The update rule is based on an optimal warping path between the current input time series and its closest prototpye. In asymmetric learning \cite{Jain2017}, the margin-growth principle always holds for the attractive force and for the repulsive force only when the optimal warping path is unique (as a necessary condition).

\subsubsection{Comments}

We conclude this sections with two remarks. 

\begin{remark}\em
Proposition \ref{prop:differentiability} states that uniqueness of an optimal warping path between $x$ and $y$ implies differentiability of $\delta_x$ at $y$. The converse statement does not hold, in general. A more general approach to arrive at Prop.~\ref{prop:differentiability} and Corollary \ref{cor:diffentiable} is as follows: First, show that a function $f$ is locally Lipschitz continuous (llc). Then invoke Rademacher's Theorem \cite{Evans1992} to assert almost-everywhere differentiability of $f$. By the rule of calculus of llc functions, we have: 
\begin{itemize} 
\itemsep0em
\item 
$\delta_x(y)$ is llc on $\S{F}^n$, because the minimum of continuously differentiable functions is llc. 
\item 
If the loss $\ell$ is llc, then $J_x(z)$ is llc, because the composition of llc functions is llc.
\qed
\end{itemize}
\end{remark}

\begin{remark}\label{remark:justification}\em
Measure-theoretic, geometric, and analytical concepts are all based on Euclidean spaces rather than DTW spaces. The reason is that contemporary learning algorithms are formulated in such a way that the current solution and input time series are first projected into the Euclidean space via optimally warping to the same length. Then an update step is performed and finally the updated solution is projected back to the DTW space. Therefore, to understand this form of learning under warping, we study the DTW distance $\delta(x, y)$ as a function restricted to the Euclidean space $\S{F}^m \times \S{F}^n$, where $m$ is the length of $x$ and $n$ is the length of $y$. 
\qed
\end{remark}

\section{Conclusion}
The multi-path set is negligible and corresponds to the union of zero sets of exponentially many quadratic forms. As a subset of the multi-path set, the multi optimal-path set is also negligible. Therefore optimal warping paths are unique almost everywhere. The implications of the proposed results are that adverse effects on learning in DTW spaces caused by non-unique optimal warping paths can be controlled and learning in DTW spaces amounts in minimizing the respective cost function by (stochastic) gradient descent almost everywhere. 

\paragraph*{\textbf{Acknowledgements}.} B.~Jain was funded by the DFG Sachbeihilfe \texttt{JA 2109/4-1}.

\small
\begin{appendix}
\section{Proofs}

The appendix presents the proof of Theorem \ref{theorem:ae} and Proposition \ref{prop:differentiability}. We first consider the univariate case ($d = 1$) in Sections \ref{subsec:01} and \ref{subsec:02}. Section \ref{subsec:01} introduces a more useful representation for proving the main results of this contribution and derives some auxiliary results. Section \ref{subsec:02} proves the proposed results for the univariate case. Finally, Section \ref{subsec:03} generalizes the proofs to the multivariate case. 

\subsection{Preliminaries}\label{subsec:01}

We assume that elements are from $\S{F} = \R$, that is $d = 1$. We write $\R^m$ instead of $\S{F}^m$ to denote a time series of length $m$. By $e^{k} \in \R^m$ we denote the $k$-th standard basis vector of $\R^{m}$ with elements
\[
e_i^k = \begin{cases}
1 & i = k\\
0 & i \neq k
\end{cases}.
\] 
\begin{definition}
Let $p = (p_1, \dots, p_L) \in \S{P}_{m,n}$ be a warping path with points $p_l = (i_l, j_l)$. Then 
\begin{align*}
\Phi &= \argsS{e^{i_1}, \ldots, e^{i_L}}{\tran} \in \R^{L \times m} \\
\Psi &= \argsS{e^{j_1}, \ldots, e^{j_L}}{\tran} \in \R^{L \times n}
\end{align*}
is the pair of \emph{embedding matrices} induced by warping path $p$. 
\end{definition}
The embedding matrices have full column rank $n$ due to the boundary and step condition of the warping path. Thus, we can regard the embedding matrices of warping path $p$ as injective linear maps $\Phi:\R^m \rightarrow \R^L$ and $\Psi:\R^n \rightarrow \R^L$ that embed time series $x \in \R^m$ and $y \in \R^n$ into $\R^L$ by matrix multiplication $\Phi x$ and $\Psi y$. We can express the cost $C_p(x,y)$ of aligning time series $x$ and $y$ along warping path $p$ by the squared Euclidean distance between their induced embeddings.
\begin{proposition}\label{prop:C=norm}
Let $\Phi$ and $\Psi$ be the embeddings induced by warping path $p \in \S{P}_{m,n}$. Then 
\[
C_p(x, y) = \normS{\Phi x - \Psi y}{^2}
\]
for all $x \in \R^m$ and all $y \in \R^n$.
\end{proposition}
\begin{proof}
\cite{Schultz2017}, Proposition A.2.
\end{proof}

\medskip

Next, we define the warping and valence matrix of a warping path. 
\begin{definition}
Let $\Phi$ and $\Psi$ be the pair of embedding matrices induced by warping path $p\in \S{P}_{m,n}$. Then the valence matrix $V\in \R^{m \times m}$ and warping matrix $W\in \R^{m \times n}$ of warping path $p$ are defined by
\begin{align*}
V &= \Phi \tran \Phi \\
W &= \Phi \tran \Psi.
\end{align*}
\end{definition}

The definition of valence and warping matrix are oriented in the following sense: The warping matrix $W \in \R^{m \times n}$ aligns a time series $y \in \R^n$ to the time axis of time series $x \in\R^m$. The diagonal elements $v_{ii}$ of the valence matrix $V \in \R^{m \times m}$ count the number of elements of $y$ warped onto the same element $x_i$ of $x$. Alternatively, we can define the \emph{complementary} valence and warping matrix of $w$ by
\begin{align*}
\overline{V} &= \Psi \tran \Psi \\
\overline{W} &= \Psi \tran \Phi = W\tran.
\end{align*}
The complementary warping matrix $\overline{W} \in \R^{n \times m}$ warps time series $x \in \R^m$ to the time axis of time series $y \in\R^n$. The diagonal elements $\overline{v}_{ii}$ of the complementary valence matrix $\overline{V} \in \R^{n \times n}$ counts the number of elements of $x$ warped onto the same element $y_i$ of $y$.

Let $p \in \S{P}_{m,n}$ be a warping path of length $L$ with induced embedding matrices $\Phi \in \R^{L \times m}$ and $\Psi \in \R^{L \times n}$. The \emph{aggregated embedding matrix} $\Theta$ induced by warping path $p$ is defined by 
\[
\Theta = \args{\Phi, -\Psi} \in \R^{L \times k},
\]
where $k = m+n$. Then the symmetric matrix $\Theta\tran \Theta$ is of the form
\[
\Theta\tran \Theta = \begin{pmatrix}
V & -W\\
-\overline{W} & \overline{V}
\end{pmatrix}.
\]
We use the following notations:
\begin{align*}
\S{X} &= \R^m \times \R^n = \R^k\\
\S{X}^2 &= \R^{k \times k}.
\end{align*}
The next result expresses the cost $C_p(x,y)$ by the matrix $\Theta\tran \Theta$.
\begin{lemma}
Let $\Theta$ be the aggregated embedding matrix induced by warping path $p \in \S{P}_{m,n}$. Then we have 
\[
C_p(z) = z\tran\Theta\tran\Theta z
\]
for all $z\in \S{X}$. 
\end{lemma}

\begin{proof}
Suppose that $\Theta = (\Phi, -\Psi)$, where $\Phi$ and $\Psi$ are the embedding matrices induced by $p$. Let $z = (x, y) \in \S{X}$. Then we have
\begin{align*}
C_p(z) 
&= \normS{\Phi x - \Psi y}{^2}\\
&= x \tran \Phi \tran \Phi x - x\tran \Phi\tran\Psi y - y\tran \Psi\tran\Phi x + y\tran \Psi\tran \Psi y\\
&= x \tran V x - x\tran W y - y\tran \overline{W} x + y \tran \overline{V} y \\
&= \args{x\tran, y\tran} \begin{pmatrix}
V & -W\\
-\overline{W} & \overline{V}
\end{pmatrix} 
\begin{pmatrix}
x \\
y
\end{pmatrix}\\
&= z\tran \Theta\tran\Theta z.
\end{align*}
\end{proof}

\medskip

\noindent
The last auxiliary result shows that the zero set of a non-zero quadratic form has measure zero.
\begin{lemma}\label{lemma:xAx-is-zero}
Let matrix $A \in \R^{n \times n}$ be non-zero and symmetric. Then 
\[
\mu\args{\cbrace{x \in \R^n \,:\, x^{\T}Ax = 0}} = 0,
\]
where $\mu$ is the Lebesgue measure on $\R^n$.
\end{lemma}

\begin{proof}
Since $A$ is symmetric, there is an orthogonal matrix $Q \in \R^{n \times n}$ such that $\Lambda = Q^{\T}AQ$ is a diagonal matrix. Consider the function
\[
f(x) = x^{\T}\Lambda x = \sum_{i=1}^n \lambda_{ii}x_i^2,
\]
where the $\lambda_{ii}$ are the diagonal elements of $\Lambda$. Since $A$ is non-zero, there is at least one $\lambda_{ii} \neq 0$. Hence, $f(x)$ is a non-zero polynomial on $\R^n$. Then the set $\S{U} = \cbrace{x \in \R^n \,:\, f(x) = 0}$ is measurable and has measure zero \cite{Caron2005}. 

We show that the set $\tilde{\S{U}} = \cbrace{\tilde{x} \in \R^n \,:\, \tilde{x}^{\T}A\tilde{x} = 0}$ is also a set of measure zero. Consider the linear map 
\[
\phi: \S{U} \rightarrow \R^n, \quad x \mapsto Q^{\T}x.
\]
First, we show that $\phi(\S{U}) = \tilde{\S{U}}$. 
\begin{itemize}
\itemsep0em
\item $\tilde{\S{U}}\subseteq \phi(\S{U})$:
Suppose that $\tilde{x} \in \tilde{\S{U}}$. With $Q\tilde{x} = x$ we have
\[
0 = \tilde{x}^{\T}A\tilde{x} = \tilde{x}^{\T}Q^{\T}\Lambda Q\tilde{x} = x^{\T}\Lambda x.
\]
This shows that $x \in \S{U}$. From $\phi(x) = Q^{\T}x = \tilde{x}$ follows that $\tilde{x} \in\phi(\S{U})$. 

\item $\phi(\S{U})\subseteq\tilde{\S{U}}$:
Let $\tilde{x} \in\phi(\S{U})$. Then $\tilde{x} = Q^{\T}x = \phi(x)$ for some $x \in \S{U}$. Hence, $x = Q\tilde{x}$ and we have
\[
0 = x^{\T}\Lambda x = \tilde{x}^{\T}Q^{\T}\Lambda Q \tilde{x} = \tilde{x}^{\T}A\tilde{x}.
\]
This shows that $\tilde{x} \in \tilde{\S{U}}$. 
\end{itemize}

Next, we show that $\mu(\tilde{\S{U}}) = 0$. Observe that the linear map $\phi$ is continuously differentiable on a measurable set $\S{U}$ with Jacobian $J_\phi(x) = Q^{\T}$. Applying \cite{Bogachev2007}, Prop.~3.7.3 gives
\[
\mu\args{\phi(\S{U})} \leq \int_{\S{U}} \abs{\det Q^{\T}}dx.
\]
Since $Q^{\T}$ is orthogonal, we have $\abs{\det Q^T} = 1$. Thus, we find that 
\[
\mu\args{\phi(\S{U})} \leq \int_{\S{U}} dx = \mu(\S{U}) = 0. 
\] 
Finally, the assertion $\mu(\tilde{\S{U}}) = 0$ follows from $\tilde{\S{U}} = \phi(\S{U})$.
\end{proof}

\subsection{Proof of Theorem \ref{theorem:ae} and Proposition \ref{prop:differentiability}}\label{subsec:02}

This section assumes the univariate case ($d = 1$). 

\medskip

\noindent
\textbf{Proof of Theorem \ref{theorem:ae}:}\\
Suppose that $\S{P}_{m,n} = \cbrace{p_1, \ldots, p_D}$. We use the following notations for all $i \in [D]$: 
\begin{enumerate}
\itemsep0em
\item $\Theta_i$ denotes the aggregated embedding matrix induced by warping path $p_i$. 
\item $V_i$ and $W_i$ are the valence and warping matrices of $p_i$.
\item $\overline{V}_i$ and $\overline{W}_i$ are the complementary valence and warping matrices of $p_i$.
\item $C_i(z)$ with $z = (x, y)$ denotes the cost $C_{p_i}(x, y)$ of aligning $x$ and $y$ along warping path $p_i$. 
\end{enumerate}
For every $i, j \in [D]$ with $i \neq j$ and for every $z = (x, y) \in \S{X}$, we have
\begin{align*}
C_i(z) - C_j(z)
= z\tran\Theta_i\tran\Theta_i z - z\tran\Theta_j\tran\Theta_j z
= z\tran A^{(ij)} z,
\end{align*}
where $A^{(ij)} \in \S{X}^2$ is a symmetric matrix of the form
\[
A^{(ij)} = \Theta_i\tran\Theta_i - \Theta_j\tran\Theta_j = \begin{pmatrix}
V_i - V_j & -W_i + W_j\\[0.5ex]
-\overline{W}_i + \overline{W}_j & \overline{V}_i - \overline{V}_j
\end{pmatrix}.
\]
For $i \neq j$ the warping paths $p_i$ and $p_j$ are different implying that the warping matrices $W_i$ and $W_j$, resp., are also different. Hence, $A^{(ij)}$ is non-zero and from Lemma \ref{lemma:xAx-is-zero} follows that $\S{U}_{ij} = \cbrace{z \in \S{X} \,:\, z \tran A^{(ij)} z = 0}$ has measure zero. Then the union 
\[
\S{U} = \bigcup_{i < j} \, \S{U}_{ij}
\]
of finitely many measure zero sets also has measure zero. It remains to show that $\S{N_X} = \S{U}$. 

\begin{itemize}
\itemsep0em
\item
$\S{N_X} \subseteq \S{U}$: Suppose that $z = (x, y) \in \S{N_X}$. Then there are indices $i, j \in [D]$ with $i < j$ such that the costs of aligning $x$ and $y$ along warping paths $p_i$ and $p_j$ are identical, that is $C_i(x, y) = C_j(x, y)$. Setting $z = (x, y)$ gives
\[
0 = C_i(z) - C_j(z) = z\tran A^{(ij)} z.
\]
Hence, $z \in \S{U}_{ij} \subseteq \S{U}$ and therefore $\S{N_X} \subseteq \S{U}$. 

\item $\S{U} \subseteq \S{N_X}$: Let $z = (x, y) \in \S{U}$. Then there is a set $\S{U}_{ij}$ containing $z$. From $C_i(z) - C_j(z) = 0$ follows that $p_i$ and $p_j$ are two warping paths between $x$ and $y$ with identical costs. Hence, $(x, y) \in \S{N_X}$. This proves the assertion.\qed
\end{itemize}

\medskip

\noindent
\textbf{Proof of Proposition \ref{prop:differentiability}:}\\
To show the proposition, we first define the notion of piecewise smooth function. A function $f:\R^n \rightarrow \R$ is piecewise smooth if it is continuous on $\R^n$ and for each $x_0 \in \R^n$ there is a neighborhood $\S{N}(x_0)$ of $x_0$ and a finite collection $\argsS{f_i}{_{i \in \S{I}}}$ of continuously differentiable functions $f_i : \S{N}(x_0) \rightarrow \R$ such that
\[
f(x) \in \cbrace{f_i(x) \,:\, i \in \S{I}}
\]
for all $x \in \S{N}(x_0)$. 

\begin{proof}
We show that the function $\delta_x(y) = \min_p C_p(x, y)$ is piecewise smooth. The function $\delta_x(y)$ is continuous, because all $C_p$ are continuous and continuity is closed under the min-operation. In addition, the functions $C_p$ are continuously differentiable as functions in the second argument. Let $y_0 \in \S{F}^n$ and let $\S{N}(y_0) \subseteq \S{F}^n$ be a neighborhood of $y_0$. Consider the index set 
\[
\S{I} = \cbrace{p \in \S{P}_{m,n} \,:\, \exists y \in \S{N}(y_0) \text{ s.t. } \delta_x(y) = C_p(x, y)}.
\]
By construction, we have $\delta_x(y) \in \cbrace{C_p(x, y) \,:\, p \in \S{I}}$ for all $y \in \S{N}(y_0)$. This shows that $\delta_x(y)$ is piecewise smooth. Then the assertion follows from \cite{Rockafellar2003}, Lemma 2. 
\end{proof}

\subsection{Generalization to the Multivariate Time Series}\label{subsec:03}

We briefly sketch how to generalize the results from the univariate to the multivariate case. The basic idea is to reduce the multivariate case to the univariate case. In the following, we assume that $x \in \S{F}^m$ and $y \in \S{F}^n$ are two $d$-variate time series and $p = (p_1, \dots, p_L) \in \S{P}_{m, n}$ is a warping path between $x$ and $y$ with elements $p_l = (i_l, j_l)$.

First observe that a $d$-variate time series $x \in \S{F}^m$ consists of $d$ individual component time series $x^{(1)}, \ldots, x^{(d)} \in \R^m$. Next, we construct the embeddings of a warping path. The $d$-variate time warping embeddings $\Phi_d: \S{F}^n \rightarrow \S{F}^L$ and $\Psi_d: \S{F}^m \rightarrow \S{F}^L$ induced by $p$ are maps of the form
\begin{align*}
\Phi_d(x) = \begin{bmatrix} 
x_{i_1} \\
\vdots \\
x_{i_L} \\
\end{bmatrix} 
,\qquad \Psi_d(y) = 
\begin{bmatrix} 
y_{j_1} \\
\vdots \\
y_{j_L} \\
\end{bmatrix}.
\end{align*}
The maps $\Phi_d$ and $\Psi_d$ can be written as
\begin{align*}
\Phi_d(x) &= \args{\Phi x^{(1)}, \ldots, \Phi x^{(d)}}\\
\Psi_d(x) &= \args{\Psi y^{(1)}, \ldots, \Psi y^{(d)}},
\end{align*}
where $\Phi$ and $\Psi$ are the embedding matrices induced by $p$. Since $\Phi$ and $\Psi$ are linear, the maps $\Phi_d$ and $\Psi_d$ are also linear maps. We show the multivariate formulation of Prop.~\ref{prop:C=norm}.
\begin{proposition}
Let $\Phi_d$ and $\Psi_d$ be the $d$-variate embeddings induced by warping path $p \in \S{P}_{m,n}$. Then 
\[
C_p(x, y) = \normS{\Phi_d(x) - \Psi_d(y)}{^2}.
\]
for all $x \in \S{F}^m$ and all $y \in \S{F}^n$.
\end{proposition}
\begin{proof}
The assertion follows from
\begin{align*}
\normS{\Phi_d(x) - \Psi_d(y)}{^2} 
&= \sum_{k=1}^d \normS{\Phi x^{(k)} - \Psi y^{(k)}}{^2} \\
&= \sum_{k=1}^d \sum_{(i,j) \in p} \argsS{x_i^{(k)}-y_j^{(k)}}{^2}\\
&= \sum_{(i,j) \in p} \sum_{k=1}^d \argsS{x_i^{(k)}-y_j^{(k)}}{^2}\\
&= \sum_{(i,j) \in p} \norm{x_i - y_j}{^2} 
= C_p(x, y).
\end{align*}
\end{proof}

Due to the properties of product spaces and product measures, the proofs of all other results can be carried out componentwise. 
\end{appendix}

\bibliographystyle{plain}

\begin{thebibliography}{99}

\bibitem{Banderier2005}
C.~Banderier and S.~Schwer.
\newblock Why Delannoy numbers? 
\newblock \emph{Journal of Statistical Planning and Inference}, 135(1):40--54, 2005.

\bibitem{Bogachev2007}
V.~Bogachev.
\newblock \emph{Measure Theory}.
\newblock Springer-Verlag Berlin Heidelberg, 2007.

\bibitem{Brill2018}
M.~Brill, T.~Fluschnik, V.~Froese, B.~Jain, R.~Niedermeier, D.~Schultz,
\newblock Exact Mean Computation in Dynamic Time Warping Spaces.
\newblock \emph{SIAM International Conference on Data Mining}, 2018 (accepted). 

\bibitem{Caron2005}
R.~Caron and T.~Traynor.
\newblock The zero set of a polynomial.
\newblock \emph{WSMR Report}, University of Windsor, 2005.

\bibitem{Cuturi2017}
M.~Cuturi and M.~Blondel. 
\newblock Soft-DTW: a differentiable loss function for time-series. 
\newblock \emph{International Conference on Machine Learning}, 2017.

\bibitem{Esling2012}
P.~Esling and C.~Agon.
\newblock Time-series data mining. 
\newblock \emph{ACM Computing Surveys}, 45:1--34, 2012.

\bibitem{Evans1992}
L.C.~Evans and R.F.~Gariepy.
\newblock \emph{Measure theory and fine properties of functions}.
\newblock CRC Press, 1992.

\bibitem{Fu2011}
T.~Fu.
\newblock A review on time series data mining. 
\newblock \emph{Engineering Applications of Artificial Intelligence}, 24(1):164--181, 2011.

\commentout{
\bibitem{Gold2017}
O.~Gold and M.~Sharir.
\newblock Dynamic Time Warping and Geometric Edit Distance: Breaking the Quadratic Barrier.
\newblock \emph{International Colloquium on Automata, Languages, and Programming}, 2017.
}

\bibitem{Halmos2013}
P.R.~Halmos.
\newblock \emph{Measure Theory}.
\newblock Springer-Verlag, 2013. 

\bibitem{Hautamaki2008}
V.~Hautamaki, P.~Nykanen, P.~Franti.
\newblock Time-series clustering by approximate prototypes.
\newblock \emph{International Conference on Pattern Recognition}, 2008.

\bibitem{Jain2015}
B.~Jain.
\newblock Generalized gradient learning on time series.
\newblock \emph{Machine Learning} 100(2-3):587--608, 2015.

\bibitem{Jain2017}
B.~Jain and D.~Schultz.
\newblock Asymmetric learning vector quantization for efficient nearest neighbor classification in dynamic time warping spaces.
\newblock \emph{Pattern Recognition}, 76:349-366, 2018.

\bibitem{Jain2017b}
B.~Jain. 
\newblock Warped-Linear Models for Time Series Classification
\newblock \emph{arXiv preprint}, arXiv:1711.09156, 2017.

\bibitem{Kohonen1998} 
T.~Kohonen and P.~Somervuo.
\newblock Self-organizing maps of symbol strings. 
\newblock \emph{Neurocomputing}, 21(1-3):19--30, 1998.

\bibitem{Kohonen2001}
T.~Kohonen.
\newblock \emph{Self-Organizing Maps}
\newblock Springer-Verlag Berlin Heidelberg, 2001.

\bibitem{Kruskal1983}
J.B.~Kruskal and M.~Liberman. 
\newblock The symmetric time-warping problem: From continuous to discrete.
\newblock \emph{Time warps, string edits and macromolecules: The theory and practice of sequence comparison}, 1983.

\commentout{
\bibitem[Mityagin, 2015]{Mityagin2015}
B.~Mityagin.
\newblock The Zero Set of a Real Analytic Function. 
\newblock \emph{arXiv preprint}, arXiv:1512.07276, 2015.
}

\bibitem{Petitjean2011}
F.~Petitjean, A.~Ketterlin, and P.~Gancarski. 
\newblock A global averaging method for dynamic time warping, with applications to clustering.
\newblock \emph{Pattern Recognition} 44(3):678--693, 2011.

\bibitem{Petitjean2016}
F.~Petitjean, G.~Forestier, G.I.~Webb, A.E.~Nicholson, Y.~Chen, and E.~Keogh.
\newblock Faster and more accurate classification of time series by exploiting a novel dynamic time warping averaging algorithm. 
\newblock \emph{Knowledge and Information Systems}, 47(1):1--26, 2016.

\bibitem{Rockafellar2003}
R.T.~Rockafellar.
\newblock A Property of Piecewise Smooth Functions.
\newblock \emph{Computational Optimization and Applications}, 25:247--250, 2003.

\bibitem{Sakoe1978}
H.~Sakoe and S.~Chiba. 
\newblock Dynamic programming algorithm optimization for spoken word recognition. 
\newblock \emph{IEEE Transactions on Acoustics, Speech, and Signal Processing}, 26(1):43--49, 1978.

\bibitem{Schultz2017}
D.~Schultz and B.~Jain. 
\newblock Nonsmooth analysis and subgradient methods for averaging in dynamic time warping spaces. 
\newblock \emph{Pattern Recognition}, 74:340--358, 2018.

\bibitem{Soheily-Khah2015}
S.~Soheily-Khah, A.~Douzal-Chouakria, and E.~Gaussier.
\newblock Progressive and Iterative Approaches for Time Series Averaging. 
\newblock \emph{Workshop on Advanced Analytics and Learning on Temporal Data}, 2015.

\bibitem{Somervuo1999} 
P. Somervuo and T. Kohonen,
\newblock Self-organizing maps and learning vector quantization for feature sequences.
\newblock \emph{Neural Processing Letters}, 10(2):151--159, 1999.

\bibitem{Xing2010}
Z.~Xing, J.~Pei, and E.~Keogh. 
\newblock A brief survey on sequence classification. 
\newblock \emph{ACM SIGKDD Explorations Newsletter}, 12(1):40--48, 2010.
\end{thebibliography}

\end{document}